\documentclass{article}

\usepackage{microtype}
\usepackage{graphicx}
\usepackage{subcaption}
\usepackage{booktabs} 
\usepackage{multirow}
\usepackage{enumitem}

\usepackage{hyperref}

\usepackage[accepted]{icml2026}

\usepackage{amsmath}
\usepackage{amssymb}
\usepackage{mathtools}
\usepackage{amsthm}

\usepackage[capitalize,noabbrev]{cleveref}

\theoremstyle{plain}
\newtheorem{theorem}{Theorem}[section]

\newtheorem{corollary}[theorem]{Corollary}
\theoremstyle{definition}
\newtheorem{definition}[theorem]{Definition}
\newtheorem{assumption}[theorem]{Assumption}
\theoremstyle{remark}
\newtheorem{remark}[theorem]{Remark}

\usepackage[textsize=tiny]{todonotes}

\icmltitlerunning{Geometry-based Schrödinger Bridges for Trustworthy Multimodal Fusion}

\begin{document}

\twocolumn[
  \icmltitle{Geometry-based Schrödinger Bridges for Trustworthy Multimodal Fusion}



  \icmlsetsymbol{cor}{*}

  \begin{icmlauthorlist}
    \icmlauthor{Jiayu Xiong}{hqu,cvlab}
    \icmlauthor{Jing Wang}{cor,hqu,cvlab}
    \icmlauthor{Qi Zhang}{cor,tongji}
    \icmlauthor{Wanlong Wang}{hqu,cvlab}
    \icmlauthor{Jun Xue}{whu}
  \end{icmlauthorlist}

  \icmlaffiliation{hqu}{Department of Computer Science and Techonology, Huaqiao University}
  \icmlaffiliation{cvlab}{Xiamen Key Laboratory of Computer Vision and Pattern Recognition, Huaqiao University}
  \icmlaffiliation{tongji}{Tongji University}
  \icmlaffiliation{whu}{School of Cyber Science and Engineering, Wuhan University}

  \icmlcorrespondingauthor{Jing Wang}{wroaring@hqu.edu.cn}
  \icmlcorrespondingauthor{Qi Zhang}{zhangqi\_cs@tongji.edu.cn}

  \icmlkeywords{Machine Learning, ICML}

  \vskip 0.3in
]



\printAffiliationsAndNotice{$^*$Corresponding authors.\\}

\begin{abstract}
Real-world multimodal systems must be robust against low-quality data, such as sensor noise, incomplete multimodal data and conflicting inputs. However, existing trustworthy fusion methods rely on the model's own prediction confidence to judge data quality. This creates a circular dependency: when a model is confident but wrong, these methods fail to detect the error. To break this loop, we propose Geometry-based Multimodal Fusion (GMF). Instead of relying on predictions, we evaluate reliability by measuring how much transport correction the input needs in latent space. We implement Diffusion Schrödinger Bridge transport with Rectified Flow, where the squared initial velocity gives an efficient learned correction score. Valid data has low squared velocity magnitude, while noisy, incomplete data or conflicting data requires stronger transport correction. This geometry-based reliability signal acts as an independent judge, effectively flagging unreliable inputs even when the classifier is fooled. Extensive experiments demonstrate that GMF significantly improves robustness against severe sensor noise and semantic conflicts compared to confidence-based baselines.
\end{abstract}

\section{Introduction}

Our perception of the world relies on multiple senses. Similarly, modern AI systems integrate diverse data, such as sensors in autonomous driving or medical scans in diagnosis~\citep{baltruaitis2019multimodal}. However, real-world data is rarely perfect. Systems frequently encounter low-quality inputs, including sensor noise or offline (incomplete data that some modalities are missing), and semantic conflicts (where different modalities contradict each other)~\citep{ma2022multimodal}. Under these conditions, standard fusion methods that simply mix features often fail~\citep{wang2020makes}. If one modality is low-quality, it spreads errors to the joint representation. To solve this, dynamic fusion has emerged as a robust alternative. By processing each input independently first, this strategy aims to isolate bad data and prevent it from ruining the final decision~\cite{zhang2024multimodal}.

The success of dynamic fusion depends entirely on one question: \textit{how do we accurately measure the quality of a modality?} This challenge has established the field of Trustworthy Multimodal Fusion~\citep{han2022trusted}. The goal is to design mechanisms that can automatically detect and down-weight unreliable data. Existing state-of-the-art methods, such as entropy-based quality-aware~\citep{zhang2023provable} and uncertainty estimation~\citep{han2022trusted}, rely on statistical reliability. They operate on a essential assumption: if a model is unsure about its prediction (e.g., the output probability is flat), the input data is likely low-quality. These methods calculate confidence scores from the classifier's output and use them to determine the fusion weights.

However, relying on the model's own output creates a dangerous circular dependency. The system uses the prediction to check if the prediction itself is reliable. This approach fails critically when the data quality is very low. Deep neural networks are known to be overconfident~\citep{guo2017calibration}. A model often makes a wrong prediction but with very high confidence~\citep{ovadia2019can}. Because the statistical method only looks at the confidence score, it cannot distinguish between a correct prediction and a confident error. As a result, the fusion mechanism fails to reject the corrupted modality, leading to system failure.

\begin{figure*}[htbp]
\centering
\begin{subfigure}[b]{0.585\textwidth}
\centering
\includegraphics[width=\linewidth]{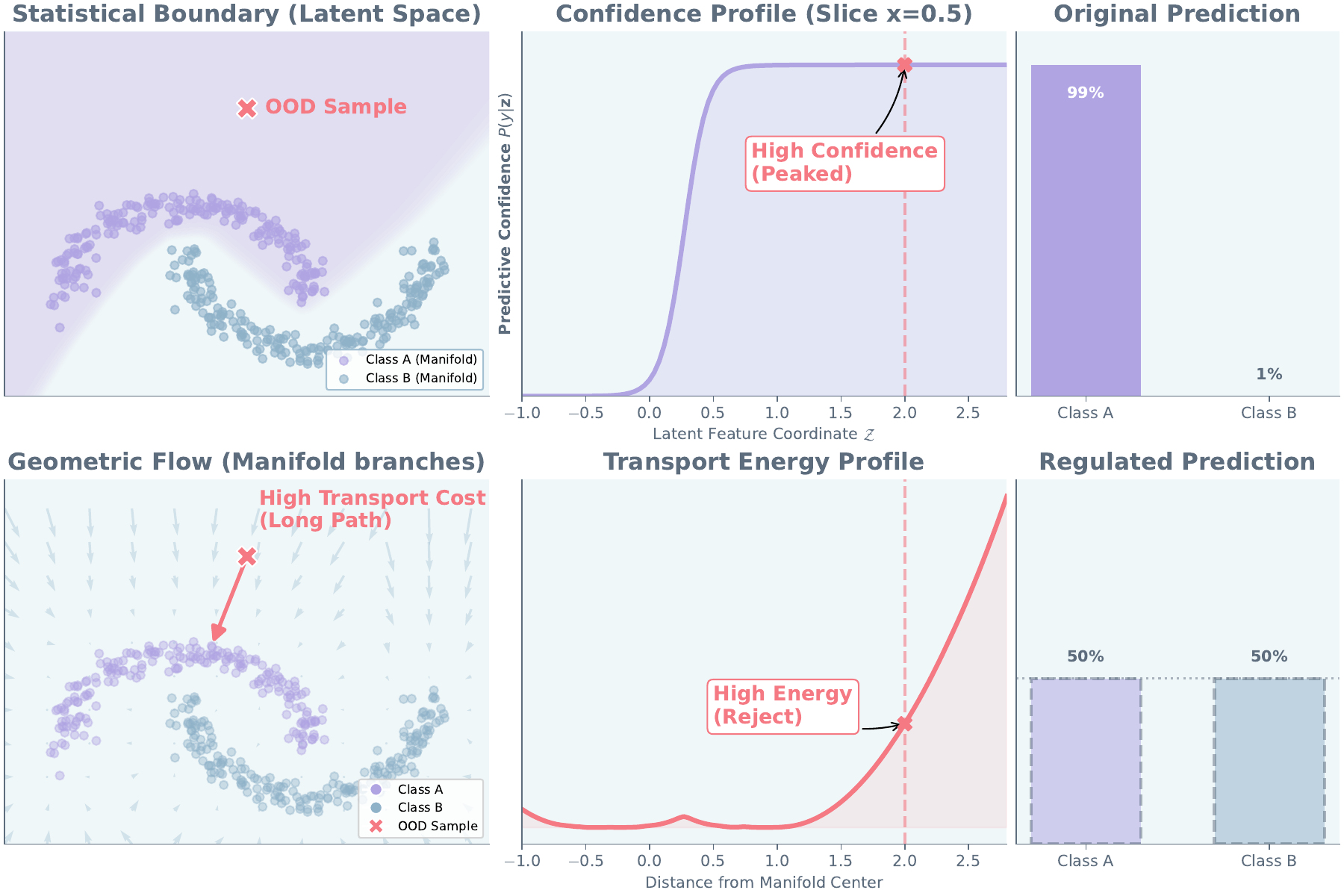}
\caption{Intra-modal Analysis}
\label{fig:intra}
\end{subfigure}
\hfill
\begin{subfigure}[b]{0.39\textwidth}
\centering
\includegraphics[width=\linewidth]{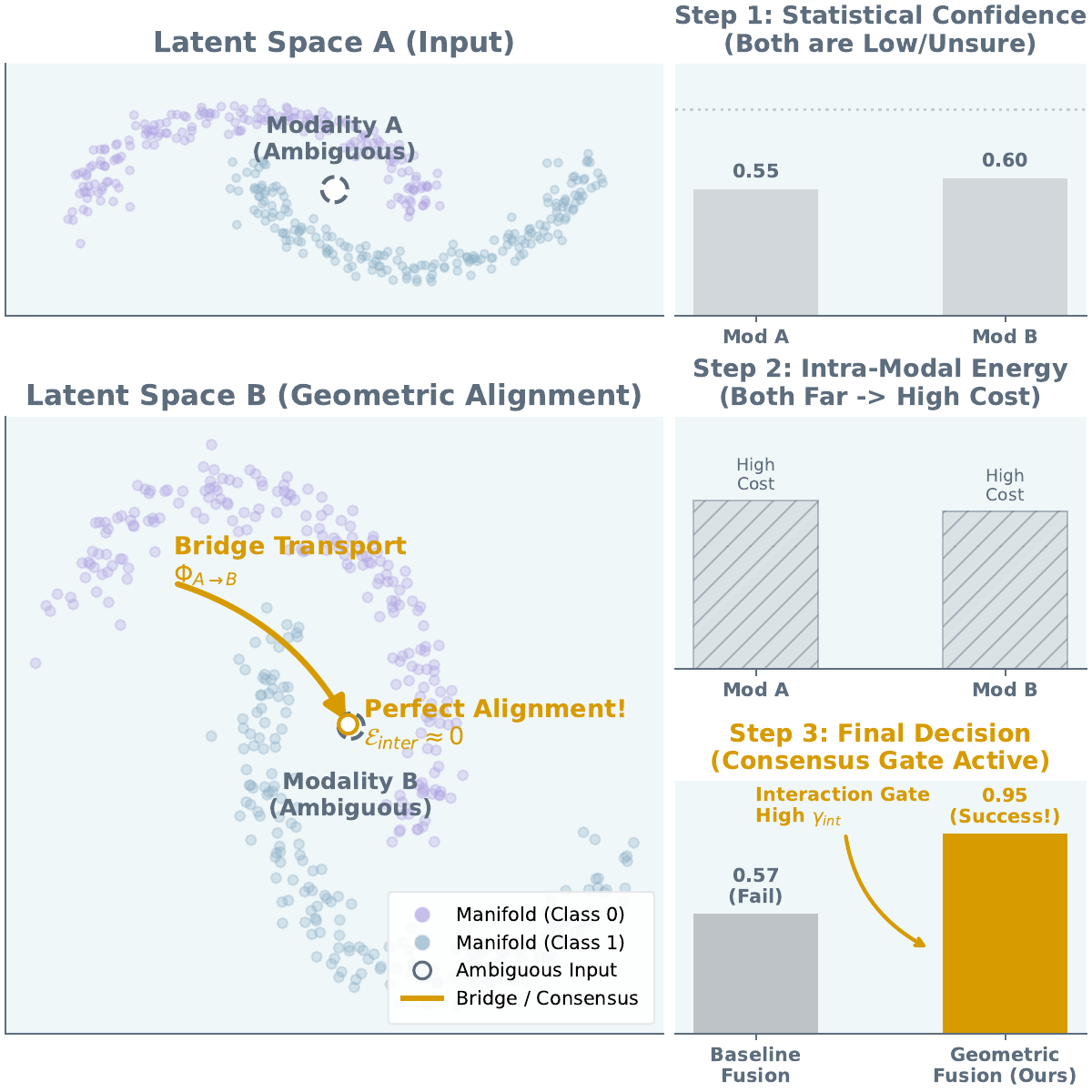}
\caption{Inter-modal Analysis}
\label{fig:inter}
\end{subfigure}
\caption{Overview of the proposed Geometry-based Multimodal Fusion (GMF). (a) \textbf{Intra-modal Analysis}: We utilize Rectified Flow to estimate how much correction an input needs in latent space. Noisy inputs or outliers require higher transport energy (squared velocity magnitude), serving as a geometric indicator of unreliability. (b) \textbf{Inter-modal Analysis}: Semantic conflicts are detected by measuring the high cross-modal transport costs between mismatched modalities, ensuring that conflicting information does not corrupt the joint decision.}
\label{fig:main_architecture}
\end{figure*}

To break this dependency, we propose a Geometry-based Multimodal Fusion (GMF). Instead of asking the classifier ``\textit{are you confident?}", we use a reliability signal computed from latent geometry rather than classifier logits. Consider clean and valid data as samples that follow a learned latent transport pattern. To implement this perspective, transport energy is estimated by the squared initial velocity, and used to measure how much correction an input needs. We employ Diffusion Schrödinger Bridge transport with Rectified Flow~\citep{liu2023flow}, which creates a straight path for this movement. The logic is simple: valid data needs little correction; noisy or conflicting data needs high correction. As shown in Fig.~\ref{fig:main_architecture}, this energy acts as an independent judge. It accurately identifies sensor noise (Intra-modal) and semantic conflicts (Inter-modal), even when the classifier is foolishly confident.
Our contributions are:
\begin{itemize}
\item \textbf{Identify} the ``circular dependency" flaw in confidence-based fusion and propose a geometric paradigm that evaluates reliability using latent transport costs.
\item \textbf{Integrate} Diffusion Schrödinger Bridge transport with Rectified Flow to create an efficient velocity-based metric. This quantifies modality quality using squared initial velocity, decoupled from the decision boundary.
\item \textbf{Derive} the fusion mechanism as the global minimizer of a geometric energy objective, providing a conditional geometric guarantee for robustness against sensor noise and semantic conflicts.
\end{itemize}
Extensive experiments on four benchmarks demonstrate that GMF significantly outperforms statistical baselines, effectively filtering out corrupted modalities in scenarios involving severe sensor noise and semantic conflicts.

\section{Related Work}

\textbf{Trustworthy Multimodal Learning.}
Trustworthy multimodal learning has evolved from static fusion strategies to dynamic, evidence-based paradigms designed to mitigate the impact of unimodal noise~\cite{han2022trusted, baltruaitis2019multimodal}. Prominent approaches like QMF~\cite{zhang2023provable} and the recent PDF~\cite{cao2024predictive} utilize evidential theory or predictive entropy to adaptively reweight modalities based on sample-level quality. To address semantic discrepancies, recent advancements (e.g., AISTATS 2025~\cite{bezirganyan2025multimodal}) have further introduced conflict-aware mechanisms that discount belief masses when modalities disagree. However, these methods fundamentally rely on the calibration of the classifier's statistical output, assuming that low predictive entropy correlates with high data quality~\cite{ma2022multimodal}, an assumption that holds in mild noise regimes but often falters under severe corruption.

\textbf{Uncertainty Quantification in Deep Learning.}
Uncertainty quantification serves as the backbone for dynamic fusion, with deterministic methods like Evidential Deep Learning (EDL)~\cite{sensoy2018evidential} and Energy-based Models~\cite{liu2020energy} being preferred for their computational efficiency over Bayesian ensembles. Despite their popularity, standard evidential regression and entropy-based metrics are widely recognized to suffer from the "overconfidence" problem~\cite{guo2017calibration}, where deep models assign high confidence scores to out-of-distribution or adversarial samples. Consequently, relying solely on these introspective statistical measures can lead to failure modes in safety-critical scenarios where the model is confident yet erroneous~\cite{kendall2017uncertainties}, prompting the need for external verification mechanisms.

\textbf{Generative Geometry and Optimal Transport.}
Beyond statistical metrics, generative models and Optimal Transport (OT) have emerged as powerful tools for characterizing the geometric structure of data manifolds. Rectified Flows~\cite{liu2023flow} and Diffusion Schrödinger Bridges (DSB)~\cite{de2021diffusion} are increasingly utilized not just for generation, but for anomaly detection by measuring reconstruction likelihoods or transport trajectory lengths~\cite{pinaya2022fast}. Similarly, Wasserstein distance estimations have been applied in domain adaptation to quantify distribution shifts~\cite{courty2016optimal}. These works demonstrate that geometric transport costs can serve as proxies for data reliability, offering a measure of deviation that is independent of the classification decision boundary.

\section{Methodology}
\label{sec:method}
We propose Geometry-based Multimodal Fusion (GMF) to assess reliability via latent transport geometry. Using Rectified Flow, we estimate the transport correction required by each sample and use these geometric costs to determine fusion weights.

\begin{table}[t]
\centering
\caption{Notation used in GMF.}
\label{tab:notation}
\small
\setlength{\tabcolsep}{8pt}
\renewcommand{\arraystretch}{1.15}
\begin{tabular}{@{}p{0.32\linewidth}p{0.56\linewidth}@{}}
\toprule
Symbol & Meaning \\
\midrule
$x^{(m)}$, $z^{(m)}$ & Input, latent feature \\
$E^{(m)}$ & Encoder \\
$v_\theta^{(m)}$ & Intra-modal velocity field \\
$v_\Phi^{(a\to b)}$ & Cross-modal velocity field \\
$\mathcal{P}_{\text{prior}}$ ($\mathcal{P}_{\text{src}}$) & Reference prior distribution \\
$\mathcal{E}_{\text{intra}}^{(m)}$ & Intra-modal transport cost \\
$\mathcal{E}_{\text{inter}}^{(a\to b)}$ & Inter-modal transport cost \\
$\beta_{\text{comp}}^{(m)}$ & Modal competition score \\
$\gamma_{\text{int}}^{(m)}$, $\tilde{\gamma}_{\text{int}}^{(m)}$ & Integration gate, stabilized gate \\
$\rho$, $\zeta$ & Conflict coefficient, exponent \\
$\epsilon_\gamma$, $\theta_r$ & Gate floor, conflict threshold \\
\addlinespace[2pt]
$\lambda_{\mathrm{geo}}$, $\lambda_{\mathrm{reg}}$ & Geometric, regularization weights\\
$w^{(m)}$, $\boldsymbol{\alpha}$ & Modal weight, evidence vector \\
$\tau$, $\kappa$, $\lambda$ & Temperature params, gate scale \\
\bottomrule
\end{tabular}
\end{table}

\subsection{Problem Formulation and Overview}
\label{sec:preliminaries}

\textbf{Problem.} We consider a multimodal classification task on a dataset $\mathcal{D} = \{(\{x_i^{(m)}\}_{m=1}^M, y_i)\}_{i=1}^N$ with $N$ samples. Each sample contains $M$ heterogeneous modalities and a label $y_i \in \mathbb R^K$. Each modality is mapped to a latent space $\mathbb{R}^d$ via an encoder $z^{(m)} = E^{(m)}(x^{(m)})$.

\textbf{Fusion Goal.} For each modality, a classifier $W_{\mathrm{cls}}^{(m)}$ predicts logits $\hat{\mathbf{y}}^{(m)}$. Our goal is to learn a dynamic weight generator $w^{(m)}(z)$ that aggregates predictions via $\hat{\mathbf{y}} = \sum_m w^{(m)} \hat{\mathbf{y}}^{(m)}$.

\textbf{Geometric Perspective.} Existing methods derive weights from logits $\hat{\mathbf{y}}^{(m)}$, creating a circular dependency. We instead generate weights directly from latent features $z^{(m)}$. This preserves intrinsic geometric information that is typically lost during the projection to class scores.

\subsection{Preliminaries: Diffusion Schrödinger Bridge}

Latent representations $z^{(m)}$ allow us to evaluate data quality via geometry. We employ the Diffusion Schrödinger Bridge (DSB) to model transport between a source distribution $\mathcal{P}_{\text{src}}$ and a target distribution $\mathcal{P}_{\text{dst}}$, avoiding the fixed Markov chains of standard diffusion. DSB optimizes the velocity field $\{v_t\}$ via
\begin{equation}
\label{eq:dsb}
\resizebox{\linewidth}{!}{
$
\min_{v_t} \int_0^1 \mathbb{E}_{z_t \sim \mathcal{P}_t} \big[\|v_t(z_t)\|^2\big] \, dt,
\ \  \text{s.t.} \ \  \partial_t \mathcal{P}_t + \nabla \cdot (\mathcal{P}_t v_t) = 0,
$
}
\end{equation}
to minimize the kinetic energy, where $t \in [0, 1]$, $\mathcal{P}_0=\mathcal{P}_{\text{src}}, \mathcal{P}_1=\mathcal{P}_{\text{dst}}$. Solving Eq.\eqref{eq:dsb} via iterative integration is too slow for real-time inference. We use Rectified Flow (RF) to linearize the transport trajectory, approximating the total cost through a single-step regression instead of integration. We parameterize the flow with a time-dependent vector field $v_\theta(z, t)$ and define a linear interpolation for $t \sim \mathcal{U}[0, 1]$:
\[
z_t = (1-t) z_0 + t z_1, \quad z_0 \sim \mathcal{P}_{\text{src}}, \quad z_1 \sim \mathcal{P}_{\text{dst}}.
\]
The network is optimized to regress the constant velocity field driving this linear path by minimizing the flow matching objective
\begin{equation}
\label{eq:rf}
\mathcal{L}_{\mathrm{RF}}(\theta) = \mathbb{E}_{t, z_0, z_1} \Big[ \big\| v_\theta(z_t, t) - (z_1 - z_0) \big\|^2 \Big].
\end{equation}
\textbf{Key Insight.} The rectified velocity field $v_\theta$ transports a source point toward a target distribution. The squared initial velocity $\|v_\theta(z, 0)\|^2$ measures the transport cost from $z$ (the source, at $t=0$) toward the target. Clean latent features cluster on the data manifold, so their transport cost reflects the cluster-to-target distance. Noisy or OOD inputs deviate from this manifold and incur higher transport energy. Since $v_\theta(z, 0)$ is evaluated at the observed feature $z$, which is the source of the flow, inference is in-distribution with training.

\subsection{Trustworthy Weight Generation}
\label{sec:quality_weight_gen}

We bridge geometric manifold learning and evidential classification. Each modality yields evidence $\mathbf{e}^{(m)} = \text{Softplus}(z^{(m)}W_{\mathrm{cls}}^{(m)})$. We modulate this evidence using scalar weights $w^{(m)}$ derived from transport dynamics.

\textbf{Intra-Modal Transport Cost.}
First, we assess the intrinsic quality of $z^{(m)}$ by measuring the transport cost toward a class-agnostic reference distribution $\mathcal{P}*{\text{prior}}$. The rectified flow transports $z^{(m)}$ toward $\mathcal{P}*{\text{prior}}$, the target of this intra-modal flow. Clean samples near the data manifold require minimal correction, while noisy or OOD samples incur higher transport energy. Using $v_\theta^{(m)}$, we define the intra-modal transport energy as a learned one-step correction score:
\begin{equation}
\label{eq:intra_cost}
    \mathcal{E}_{\text{intra}}^{(m)} = \big\| v_\theta^{(m)}(z^{(m)}, 0) \big\|_2^2.
\end{equation}
The field $v_\theta^{(m)}$ is calibrated on clean latent transport paths. We assume that larger squared velocity magnitude indicates larger deviation from the clean latent manifold. Under this assumption, low-quality inputs such as noisy samples trigger stronger learned transport correction. This makes $\mathcal{E}_{\text{intra}}^{(m)}$ a geometric detector for sensor failure and outliers.

\textbf{Inter-Modal Transport Cost.}
Second, we evaluate semantic consistency by transporting modality $a$ to the manifold of modality $b$. We implement the cross-modal mapping via the direction-specific cross-modal velocity field $v_{\Phi}^{(a \to b)}$. Each direction $a\to b$ has its own velocity field, and inference starts from the source feature $z^{(a)}$. Thus the source point enters through the initial state rather than through an additional explicit conditioning input. We define the induced one-step map as:
\[
\Phi_{a\to b}(z) := z + v_{\Phi}^{(a \to b)}(z, 0).
\]
The projected source feature and inter-modal cost are:
\begin{equation}
\label{eq:inter_cost}
\hat{z}^{(a \to b)}=\Phi_{a\to b}(z^{(a)}),
\quad
\mathcal{E}_{\text{inter}}^{(a \to b)}
=
\big\|\Phi_{a\to b}(z^{(a)})-z^{(b)}\big\|_2^2 .
\end{equation}
A large discrepancy $\mathcal{E}_{\text{inter}}^{(a \to b)}$ indicates that the source $a$ cannot be geometrically aligned with observation $b$. This high cross-modal transport cost explicitly reveals a semantic conflict between the modalities.

\textbf{Geometric Competition and Interaction.}
We synthesize these geometric costs into the final weight $w^{(m)}$ through a competition-interaction mechanism. 

\textbf{(1) Competition:} We prioritize modalities with high intrinsic stability using a Boltzmann distribution over transport energy, yielding a base score $\beta_{\text{comp}}^{(m)}$:
\begin{equation}
\label{eq:competition}
    \beta_{\text{comp}}^{(m)} = \frac{e^{ - \mathcal{E}_{\text{intra}}^{(m)} / \tau }} {\sum_{k=1}^M e^{ - \mathcal{E}_{\text{intra}}^{(k)} / \tau }}.
\end{equation}
\textbf{(2) Interaction:} We gate this score according to the geometric support that modality $m$ receives from stable neighbors. The interaction gate $\gamma_{\text{int}}^{(m)}$ aggregates cross-modal consensus from reliable modalities:
\begin{equation}
\label{eq:interaction}
    \gamma_{\text{int}}^{(m)}
    = \lambda \sum_{k \neq m} r^{(k)}
    \exp\left(-\frac{\mathcal{E}_{\text{inter}}^{(k \to m)}}{\kappa}\right),
\end{equation}
where $r^{(k)} = \sigma(\theta_r - \mathcal{E}_{\text{intra}}^{(k)})$. Here $\kappa>0$ controls the sensitivity of the interaction gate to cross-modal transport cost.
For numerical stability, we use the stabilized gate $\tilde{\gamma}_{\text{int}}^{(m)}=\gamma_{\text{int}}^{(m)}+\epsilon_\gamma$, where $\epsilon_\gamma>0$ is a small numerical floor.
The final fusion weights and global Dirichlet parameters are computed as:
\begin{equation}
\label{eq:final_weight}
    w^{(m)} = \frac{\beta_{\text{comp}}^{(m)} \tilde{\gamma}_{\text{int}}^{(m)}} {\sum_{j} \beta_{\text{comp}}^{(j)} \tilde{\gamma}_{\text{int}}^{(j)}}, \quad \boldsymbol{\alpha} = \sum_{m=1}^M w^{(m)} \mathbf{e}^{(m)} + \mathbf{1}.
\end{equation}

\subsection{Optimization via Unified Objective}
\label{sec:joint_optimization}

We train the geometric branch and the evidential decision branch in a unified framework with separated gradient paths. Unlike multi-stage approaches, this strategy keeps the geometry-derived reliability signals aligned with the decision branch while preventing each loss from updating unrelated modules. The total objective comprises a geometric transport loss, an evidential task loss, and a conflict-aware regularization term.

\textbf{Geometric Transport Loss.}
To instantiate the reliability metrics, we optimize both the intra-modal and inter-modal transport dynamics using flow matching objectives. We formulate the unified geometric objective via
\begin{equation}
\label{eq:loss_geo}
    \mathcal{L}_{\text{geo}}
    =
    \sum_{m=1}^M
    \mathcal{L}_{\text{intra}}^{(m)}
    +
    \sum_{a \neq b}
    \mathcal{L}_{\text{inter}}^{(a \to b)}.
\end{equation}
For the intra-modal term, we adopt the Rectified Flow objective defined in \eqref{eq:rf} to simplify the transport cost estimation:
\begin{equation}
\label{eq:loss_intra}
    \mathcal{L}_{\text{intra}}^{(m)} = \mathcal{L}_{\text{RF}}(\theta^{(m)}),
\end{equation}
where $\theta^{(m)}$ denotes the parameters of the intra-modal velocity network. This term regresses the velocity field toward the straight trajectory that transports a latent feature to the reference distribution. For the inter-modal term, we align geometric structures by solving a direction-specific boundary value problem between modalities. We minimize the flow matching error via
\begin{equation}
\label{eq:loss_inter}
    \mathcal{L}_{\text{inter}}^{(a \to b)}
    =
    \mathbb{E}_{t, z^{(a)}, z^{(b)}}
    \Big[
    \big\|
    v_{\Phi}^{(a \to b)}(z_t, t) - (z^{(b)} - z^{(a)})
    \big\|_2^2
    \Big],
\end{equation}
where $v_{\Phi}^{(a \to b)}$ is the modality-pair-specific cross-modal velocity network parameterized by $\Phi$. Each direction $a\to b$ has its own velocity field, and the source point enters through the initial state on the interpolation path rather than through an additional explicit conditioning input. Here we define the interpolation path $z_t = (1-t)z^{(a)} + t z^{(b)}$ sampled at time $t \sim \mathcal{U}[0, 1]$. Minimizing this objective encourages the velocity field to capture the transport correction needed to align distinct semantic manifolds.

\textbf{Evidential Task Loss.}
We optimize classification performance using the fused Dirichlet parameters $\boldsymbol{\alpha}$ derived from the geometric weights. Following the evidential learning paradigm, we minimize the expected Bayes risk via
\begin{equation}
\label{eq:loss_task}
    \mathcal{L}_{\text{task}}
    =
    \sum_{i=1}^N
    \left(
    \mathcal{L}_{\text{eCE}}(\boldsymbol{\alpha}_i, y_i)
    +
    \lambda_{\text{KL}}\,
    \mathrm{KL}\big(
    \mathrm{Dir}(\boldsymbol{\alpha}_i)
    \;\|\;
    \mathrm{Dir}(\mathbf{1})
    \big)
    \right),
\end{equation}
where $\mathcal{L}_{\text{eCE}}$ is the expected cross-entropy loss and $\mathrm{Dir}(\mathbf{1})$ represents the uniform Dirichlet distribution. This term penalizes evidential overconfidence when geometric support is absent. It forces the predictive distribution toward a neutral state if the aggregated evidence is insufficient.

\textbf{Conflict-Aware Regularization.}
We introduce a regularization term to mitigate semantic conflicts where modalities are confident but geometrically disjoint. We define a global consistency coefficient $\rho$ as a lightweight summary of inter-modal geometric deviation:
\begin{equation}
\label{eq:rho}
\rho = \frac{1}{M(M-1)} \sum_{a \neq b} \exp\left( -\frac{\mathcal{E}_{\mathrm{inter}}^{(a \to b)}}{\kappa} \right).
\end{equation}
$\mathcal{E}_{\mathrm{inter}}^{(a \to b)}$ is the pairwise endpoint residual used by the fusion gate, while $\rho$ is a lightweight global summary used only by the auxiliary regularizer. They play different roles: $\mathcal{E}_{\mathrm{inter}}$ controls sample-level cross-modal weighting, whereas $\rho$ controls the strength of uniformity regularization when global cross-modal deviation is large.
When $\rho$ approaches zero, it indicates significant manifold deviation across modalities. In this case, the regularization forces the predictive distribution toward uniformity via
\begin{equation}
\label{eq:loss_reg}
    \mathcal{L}_{\text{reg}}
    =
    (1 - \rho)^{\zeta}\,
    \mathrm{KL}\big(
    \mathrm{Dir}(\boldsymbol{\alpha})
    \;\|\;
    \mathrm{Dir}(\mathbf{1})
    \big),
\end{equation}
where $\zeta \ge 1$ is a sensitivity hyperparameter. The total training objective is the weighted sum
\begin{equation}
\label{eq:loss_total}
\mathcal{L}_{\text{total}} = \mathcal{L}_{\text{task}} + \lambda_{\text{geo}} \mathcal{L}_{\text{geo}} + \lambda_{\text{reg}} \mathcal{L}_{\text{reg}},
\end{equation}
where $\lambda_{\text{geo}}$ and $\lambda_{\text{reg}}$ control the strength of geometric constraints. In training, this objective is implemented with the separated gradient paths summarized in Algorithm~\ref{alg:train_gmf}.

\begin{algorithm}[t]
\caption{Training Procedure of GMF}
\label{alg:train_gmf}
\begin{algorithmic}[1]
\STATE \textbf{Input:} $\mathcal{D}$, $\{E^{(m)}\}_{m=1}^M$, $\mathcal{P}_{\text{prior}}$
\STATE \textbf{Init:} $\{v_\theta^{(m)}\}_{m=1}^M$, $\{v_\Phi^{(a\to b)}\}_{a\neq b}$, heads
\FOR{each mini-batch $(\{x_i^{(m)}\}_{m=1}^M,y_i)_{i=1}^B \sim \mathcal{D}$}
    \FOR{each modality $m$}
        \STATE $z_i^{(m)} \leftarrow E^{(m)}(x_i^{(m)}),\quad \bar z_i^{(m)}\leftarrow \mathrm{sg}(z_i^{(m)})$
        \STATE $t\sim\mathcal{U}[0,1],\quad z_{0,i}^{(m)} \sim \mathcal{P}_{\text{prior}},\quad z_{t,i}^{(m)}=(1-t)\bar z_i^{(m)}+t z_{0,i}^{(m)}$
        \STATE $\mathcal{L}_{\text{intra}}^{(m)}$ by Eq.~\eqref{eq:loss_intra}; $\mathcal{E}_{\text{intra}}^{(m)}$ by Eq.~\eqref{eq:intra_cost}
    \ENDFOR
    \FOR{each pair $a\neq b$}
        \STATE $t\sim\mathcal{U}[0,1],\quad z_{t,i}^{(a,b)}=(1-t)\bar z_i^{(a)}+t\bar z_i^{(b)}$
        \STATE $\mathcal{L}_{\text{inter}}^{(a\to b)}$ by Eq.~\eqref{eq:loss_inter}; $\mathcal{E}_{\text{inter}}^{(a\to b)}$ by Eq.~\eqref{eq:inter_cost}
    \ENDFOR
    \STATE $\{v_\theta,v_\Phi\}\leftarrow \mathrm{opt}_{\text{geo}}(\mathcal{L}_{\text{geo}})$
    \STATE $w^{(m)}$ by Eqs.~\eqref{eq:competition}--\eqref{eq:final_weight} with $\mathrm{sg}(\mathcal{E}_{\text{intra}}),\mathrm{sg}(\mathcal{E}_{\text{inter}})$
    \STATE $\boldsymbol{\alpha}_i$ by Eq.~\eqref{eq:final_weight}; $\mathcal{L}_{\text{task}}$ by Eq.~\eqref{eq:loss_task}
    \STATE $\rho$ by Eq.~\eqref{eq:rho}; $\mathcal{L}_{\text{reg}}(\mathrm{sg}(\rho),\boldsymbol{\alpha})$ by Eq.~\eqref{eq:loss_reg}
    \STATE $\{E,\text{heads}\}\leftarrow \mathrm{opt}_{\text{task}}(\mathcal{L}_{\text{task}}+\lambda_{\text{reg}}\mathcal{L}_{\text{reg}})$
\ENDFOR
\STATE \textbf{Output:} GMF parameters
\end{algorithmic}
\end{algorithm}

\section{Theoretical Framework}
\label{sec:theory}

We establish conditional theoretical results for the proposed framework. We first define the geometric regularity of the latent space. We then prove the optimality of our fusion mechanism. Finally, we derive safety bounds for scenarios involving semantic conflicts.

\subsection{Geometric Structure of Latent Representations}

We formalize the manifold hypothesis utilized in Sec.\ref{sec:method}, characterize the latent space through two regularity conditions.

\begin{assumption}[Latent Space Regularity]
\label{ass:regularity}
Let $z^{(m)} \sim P(\cdot | y=k)$ denote a latent representation from 
modality $m$ and class $k$. We assume there exist modality-specific class manifolds $\{\mathcal{M}_k^{(m)}\}_{k=1}^K$ such that:
\begin{enumerate}[label=(\roman*), leftmargin=2em]
    \item \textbf{Concentration}: With high probability, 
    $\mathrm{dist}(z^{(m)}, \mathcal{M}_k^{(m)}) \le \epsilon$ for some 
    $\epsilon > 0$. This reflects the local invariance of well-trained 
    features.
    \item \textbf{Metric Separability}: Distinct structures are 
    separated by a margin $\delta > 2\epsilon$: for any modality $m$ and any $i \neq j$, 
    $\inf_{u \in \mathcal{M}_i^{(m)}, v \in \mathcal{M}_j^{(m)}} \|u - v\|_2 \ge \delta$.
\end{enumerate}
\end{assumption}

\begin{remark}
Condition (i) is standard in representation learning and emerges from 
contrastive objectives \citep{CLIP}. Condition (ii) 
ensures that ``uncertainty balls" $\mathcal{B}(\mathcal{M}_k^{(m)}, \epsilon)$ 
around each class manifold are disjoint. Under this regularity, any 
confident misclassification—where $z^{(m)}$ lies near $\mathcal{M}_k^{(m)}$ 
for $k \neq y$—is at least $(\delta - \epsilon)$-away from the correct 
manifold $\mathcal{M}_y^{(m)}$, providing a detectable geometric signal 
independent of prediction confidence.
\end{remark}

\subsection{Optimal Fusion via Geometric Reliability}

We analyze the fusion weights as the solution to a constrained optimization problem. We first define the scalar objective function for each modality.

\begin{definition}[Effective Geometric Cost]
We utilize the intra-modal transport cost $\mathcal{E}_{\mathrm{intra}}^{(m)}$ from Eq.\eqref{eq:intra_cost} and the stabilized interaction gate $\tilde{\gamma}_{\mathrm{int}}^{(m)}$ from Eq.\eqref{eq:final_weight}. We define the effective geometric cost $\mathcal{C}^{(m)} \in \mathbb{R}$ as:
\begin{equation}
    \mathcal{C}^{(m)} := \mathcal{E}_{\mathrm{intra}}^{(m)} - \tau \ln \tilde{\gamma}_{\mathrm{int}}^{(m)},
\end{equation}
where $\tau > 0$ is the temperature parameter and $\tilde{\gamma}_{\mathrm{int}}^{(m)}=\gamma_{\mathrm{int}}^{(m)}+\epsilon_\gamma$.
Here $\mathcal{C}^{(m)}$ is a calibrated effective cost rather than a literal conserved physical energy. The temperature $\tau$ maps the log-gate term into the same cost scale as $\mathcal{E}_{\mathrm{intra}}^{(m)}$.
\end{definition}

This cost synthesizes intrinsic quality and extrinsic support into a single scalar. We now state the optimality theorem.

Let $\Delta^{M-1} = \{w \in \mathbb{R}^M_{\ge 0} : \sum_{m} w^{(m)} = 1\}$ denote the probability simplex, and let $H(w) = -\sum_{m} w^{(m)} \ln w^{(m)}$ be the Shannon entropy.

\begin{theorem}[Optimality of Fusion Weights]
\label{thm:optimality}
For a fixed sample and fixed geometric costs $\mathcal{C}^{(1)},\ldots,\mathcal{C}^{(M)}$, the entropy-regularized minimization problem:
\begin{equation}
    \min_{w \in \Delta^{M-1}} \;\; \sum_{m=1}^M w^{(m)} \mathcal{C}^{(m)} - \tau H(w),
\end{equation}
admits a unique minimizer $w^*$, given by the Gibbs distribution:
\begin{equation}
    w^{*(m)} = \frac{ \exp \big( -\mathcal{C}^{(m)} / \tau \big) }{ \sum_{j=1}^M \exp \big( -\mathcal{C}^{(j)} / \tau \big) }.
\end{equation}
\end{theorem}

The proof is provided in the Appendix. For fixed geometric costs, this result justifies Eq.\eqref{eq:final_weight} as the entropy-regularized minimizer over the fusion weights.

\subsection{Robustness under Semantic Conflicts}

We analyze the system behavior when modalities contradict each other. We consider a case where Modality $A$ is correct, but Modality $B$ confidently predicts a wrong class.
The following result is conditional on local cross-modal consistency. In practice, Eq.~\eqref{eq:loss_inter} learns the direction-specific map from matched multimodal pairs, and the empirical separation in Fig.~\ref{fig:expr_vis}(b) supports this behavior.

\begin{theorem}[Geometric Barrier Principle]
\label{thm:barrier}
Assume Assumption~\ref{ass:regularity} and suppose that the cross-modal map $\Phi_{n\to B}$ is $\xi$-semantically consistent with $\xi \le \epsilon$, i.e., $\mathrm{dist}(\Phi_{n\to B}(u), \mathcal{M}_y^{(B)}) \le \xi$ for any $u\in\mathcal{B}(\mathcal{M}_y^{(n)},\epsilon)$. If $z^{(n)}$ encodes class $y$ and $z^{(B)}$ encodes a conflicting class $k\neq y$, then
\begin{equation}
    \mathcal{E}_{\mathrm{inter}}^{(n \to B)} 
    = \| \Phi_{n \to B}(z^{(n)}) - z^{(B)} \|_2^2 
    \ge (\delta - 2\epsilon)^2.
\end{equation}
\end{theorem}

Under latent separability and local cross-modal consistency, this result shows that semantic conflicts induce a high geometric cost, exponentially suppressing the interaction gate and unnormalized contribution of the incorrect modality (Corollary~\ref{cor:suppression}). In the two-modality case, the reliable neighbor $n$ can be $A$. The formal proof is detailed in Appendix~\ref{app:theory}.

\begin{corollary}[Exponential Suppression of Conflicting Modalities]
\label{cor:suppression}
Under the conditions of Theorem~\ref{thm:barrier}, the inter-modal 
interaction gate for the conflicting modality $B$ satisfies the following bound. Let $D = (\delta - 2\epsilon)^2$. If modality $B$ conflicts with all active reliable neighbors, then
\begin{equation}
    \gamma_{\mathrm{int}}^{(B)} 
    \le \lambda (M-1) \exp\left( -\frac{D}{\kappa} \right).
\end{equation}
Consequently, the stabilized gate satisfies $\tilde{\gamma}_{\mathrm{int}}^{(B)} \le \lambda (M-1)\exp(-D/\kappa)+\epsilon_\gamma$.

\end{corollary}

Here, an active reliable neighbor means a modality included in the interaction sum whose feature satisfies the reliability and semantic-consistency assumptions of Theorem~\ref{thm:barrier}. The soft gate $r^{(n)}$ measures its reliability strength, with $0\le r^{(n)}\le 1$. Consequently, the unnormalized contribution of $B$ is exponentially suppressed up to the numerical floor, reducing its normalized weight when the denominator contains reliable non-conflicting evidence with non-vanishing contribution. Proof is provided in Appendix~\ref{app:proof_corollary}.

Corollary~\ref{cor:suppression} gives a qualitative exponential suppression trend. $(\delta - 2\epsilon)^2$ acts as a geometric barrier. Even if the conflicting modality $B$ has low intra-modal cost (i.e., overconfidence), weak cross-modal consensus drives the stabilized gate toward the numerical floor, suppressing its unnormalized contribution and increasing the rejection tendency.

\section{Experiments}
\label{sec:experiments}

We evaluate the proposed Geometry-based Multimodal Fusion (GMF) framework across four diverse benchmarks. \textbf{Sec.~\ref{subsec:setup}} details the experimental setup and implementation. \textbf{Sec.~\ref{subsec:robustness}} verifies robustness against sensor noise and offline. \textbf{Sec.~\ref{subsec:conflict}} tests safety under semantic conflicts. \textbf{Sec.~\ref{subsec:medical}} examines risk stratification in medical diagnosis. Finally, \textbf{Sec.~\ref{subsec:ablation}} provides ablation studies and efficiency analysis. Implementation are detailed in Appendix~\ref{app:implementation}.

\subsection{Experimental Setup}
\label{subsec:setup}

\textbf{Evaluation Protocols.} 
We design protocols to dissociate predictive confidence from data validity. This allows us to strictly test the ``circular dependency" hypothesis. Standard accuracy often fails to reveal when a model is ``right for the wrong reasons." Therefore, we implement three stress tests:
(1) \textbf{Sensor Corruption}: We inject Gaussian noise to trigger classifier overconfidence. We verify if transport energy identifies these manifold deviations.
(2) \textbf{Semantic Conflict}: We mismatch modalities to create contradictory inputs. This assesses the ability to prioritize safety over blind belief.
(3) \textbf{Incomplete Multimodal Data}: We randomly drop input modalities. We assess if GMF identifies null signals as high-cost outliers to reweight the remaining reliable sources.

\textbf{Datasets.} 
We employ four benchmarks covering diverse modalities:
\textbf{NYU Depth V2} (RGB-D) tests fusion across heterogeneous geometric structures~\citep{silberman2012indoor}.
\textbf{UPMC FOOD-101} (Image-Text) evaluates robustness against real-world web noise~\citep{wang2015recipe}.
\textbf{MVSA-Single} (Image-Text) provides a testbed for semantic sentiment conflicts~\citep{niu2016sentiment}.
\textbf{PneumoniaMNIST} (X-ray/Report) serves as a high-stakes medical benchmark~\citep{yang2023medmnist}.

\textbf{Baselines.}
We compare GMF against the baselines summarized in Tab.~\ref{tab:baselines}. These baselines cover both feature-level and decision-level fusion strategies to provide a thorough evaluation of reliability assessment.

\begin{table}[t!]
\centering
\small
\caption{Overview of baselines. We lists methods by venue/year, and fusion type (FL: Feature-level, DL: Decision-level).}
\label{tab:baselines}
\begin{tabular}{l|l|c}
\toprule
\textbf{Method} & \textbf{Venue / Year} & \textbf{Type} \\
\midrule
Concat & - & FL \\
Late Fusion & - & DL \\
MMTM \cite{joze2020mmtm} & CVPR 2020 & FL \\
TMC \cite{han2022trusted} & NeurIPS 2021 & DL \\
QMF \cite{zhang2023provable} & ICML 2023 & DL \\
MLA \cite{zhang2024multimodal} & CVPR 2024 & FL/Tr. \\
PDF \cite{cao2024predictive} & ICML 2024 & DL \\
DBF \cite{bezirganyan2025multimodal} & AISTATS 2025 & DL \\
UAW-EEF \cite{guo2025uncertainty} & ICASSP 2025 & DL \\
GOMFuNet \cite{guo2025gomfunet} & Mathematics 2025 & FL \\
\textbf{GMF (Ours)} & \textbf{Proposed} & \textbf{DL} \\
\bottomrule
\end{tabular}
\vskip -1em
\end{table}

\subsection{Challenge I: Robustness against Sensor Corruption}
\label{subsec:robustness}

\textbf{Hypothesis.} Statistical fusion methods rely on classifier confidence to detect errors. This creates a circular dependency because deep classifiers often assign high confidence to noisy inputs. We hypothesize that GMF breaks this dependency by using geometric transport energy. This geometry-based reliability signal should increase with noise intensity regardless of the internal state of the classifier.

\textbf{Protocol.} We simulate failure by injecting Gaussian noise into the image modality of NYU Depth V2 and UPMC Food-101 datasets. We use standard deviations $\sigma \in {1.0, 2.0}$. We evaluate classification accuracy and the ability of the model to distinguish reliable data from corrupted samples.

\begin{table}[t]
\centering
\caption{Classification accuracy (\%) under diverse low-quality scenarios. We report performance on clean data, two levels of Gaussian noise (added to images), and incomplete data input (Inco, 50\% text/depth features masked) cases. GMF exhibits superior robustness across all corruptions by identifying high transport energy for both noisy and incomplete inputs. The best in each column is in \textbf{bold}. Second-best results are \underline{underlined}.}
\label{tab:noise_robustness}
\resizebox{\linewidth}{!}{
\begin{tabular}{l|c|ccc|c|ccc}
\toprule
\multirow{3}{*}{Method} & \multicolumn{4}{c|}{NYU Depth V2 (RGB-D)} & \multicolumn{4}{c}{UPMC Food-101 (Img-Text)} \\
\cline{2-9}
 & \multirow{2}{*}{Clean} & \multicolumn{2}{c}{Noise ($\sigma$)} & \multirow{2}{*}{Inco} & \multirow{2}{*}{Clean} & \multicolumn{2}{c}{Noise ($\sigma$)} & \multirow{2}{*}{Inco} \\
 & & 1.0 & 2.0 & & & 1.0 & 2.0 & \\
\midrule
Concat & 68.5 & 42.1 & 28.4 & 35.8 & 89.4 & 45.3 & 30.2 & 41.2 \\
Late Fusion & 69.1 & 51.9 & 38.6 & 48.2 & 90.7 & 58.0 & 41.5 & 62.1 \\
MMTM \cite{joze2020mmtm} & 70.2 & 53.8 & 40.5 & 50.8 & 91.2 & 59.4 & 43.1 & 64.8 \\
TMC \cite{han2022trusted} & 71.0 & 53.3 & 41.2 & 52.1 & 91.5 & 56.7 & 39.5 & 68.2 \\
QMF \cite{zhang2023provable} & 71.2 & 55.6 & 45.8 & 56.4 & 92.9 & 62.2 & 48.6 & 78.5 \\
MLA \cite{zhang2024multimodal} & 72.1 & 50.8 & 35.4 & 42.1 & \textbf{93.6} & 55.4 & 38.9 & 64.2 \\
PDF \cite{cao2024predictive} & \textbf{72.5} & 57.1 & 47.5 & 58.2 & \underline{93.4} & 64.5 & 51.2 & 80.6 \\
DBF \cite{bezirganyan2025multimodal} & \underline{72.3} & 58.0 & 49.1 & 60.3 & 93.2 & 65.1 & 52.4 & 81.3 \\
UAW-EEF \cite{guo2025uncertainty} & 71.8 & \underline{58.4} & \underline{50.2} & \underline{61.5} & 92.8 & \underline{66.4} & \underline{53.1} & \underline{82.4} \\
GOMFuNet \cite{guo2025gomfunet} & 70.5 & 54.1 & 42.6 & 51.9 & 91.6 & 60.2 & 44.5 & 66.1 \\
\textbf{GMF (Ours)} & 71.9 & \textbf{61.4} & \textbf{55.2} & \textbf{64.8} & 93.1 & \textbf{69.8} & \textbf{58.7} & \textbf{85.4} \\
\bottomrule
\end{tabular}
}
\end{table}

Tab.~\ref{tab:noise_robustness} shows that GMF remains competitive on clean data while exhibiting stronger resilience as noise increases. The advantage is most visible in the severe corruption setting, where confidence-based baselines degrade more sharply. Error bars are reported in Appendix~\ref{app:bar}.

GMF also generalizes to incomplete data scenarios. When a modality is dropped, the input deviates from the clean latent data distribution ($P_{\text{dst}}$), causing the velocity network $v_\theta$ to output vectors with high squared velocity magnitude. The observed increase in transport energy directly signals structural failure, which helps GMF down-weight incomplete or corrupted inputs.

\textbf{Decoupling Reliability from Confidence.} 
We analyze the reliability diagrams in Figure~\ref{fig:expr_vis}(a) to verify the resolution of the circular dependency. Statistical baselines exhibit a clear calibration gap under severe noise: their predicted confidence remains high even as accuracy drops. This confirms that internal uncertainty estimates collapse under corruption. In contrast, the GMF curve tracks the ideal diagonal more closely. The reliability diagram evaluates the calibration of the final fused prediction rather than the transport energy alone, and the smaller calibration gap indicates that geometry-guided weighting helps the evidential classifier avoid overconfident predictions under severe corruption.

\begin{figure*}
    \centering
    \includegraphics[width=1\linewidth]{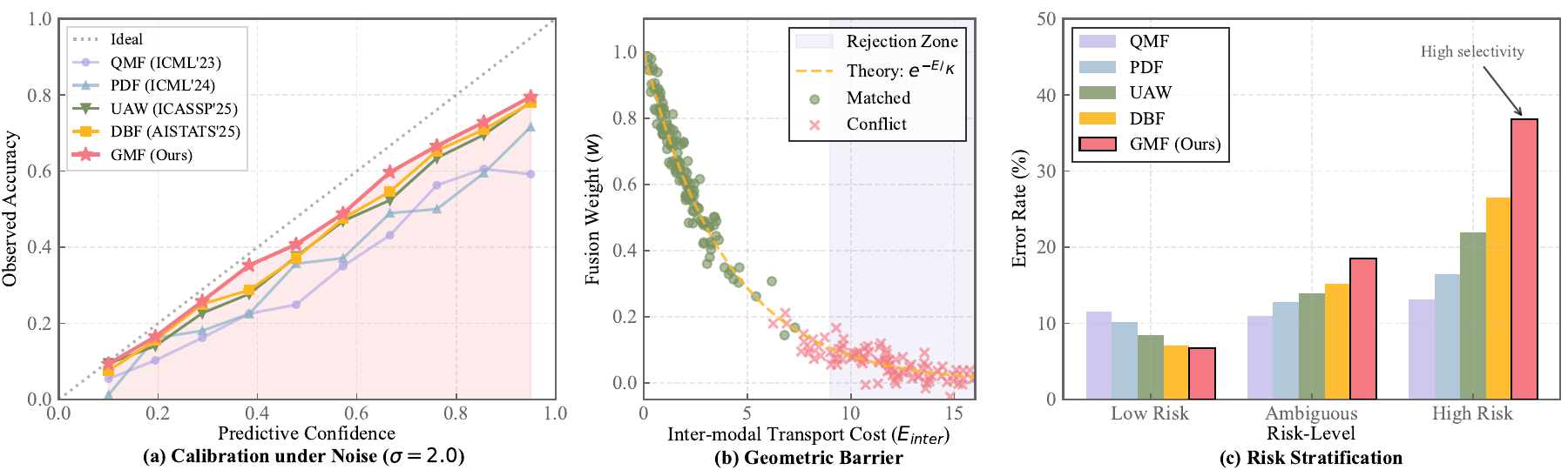}
    \caption{Quantitative analysis of the GMF framework. Left: Reliability diagrams show that GMF maintains superior calibration under noise, while baselines exhibit significant overconfidence. Middle: The exponential trend of fusion weights relative to inter-modal cost validates the Geometric Barrier Principle for conflict detection, up to a constant scale. Right: GMF effectively concentrates model errors within high-energy samples, enabling reliable risk stratification for safety-critical tasks.}
    \label{fig:expr_vis}
    \vskip -0.5em
\end{figure*}

\subsection{Challenge II: Safety under Semantic Conflict}
\label{subsec:conflict}

Safety-critical systems must distinguish between valid data and semantic conflicts. We validate the ability of GMF to break the circular dependency using a Confident Conflict protocol. This tests if the model can detect when confident modalities contradict each other.

\textbf{Conflict Protocol.} 
We construct a Conflict Subset on MVSA-Single by shuffling image-text pairs. This forces a semantic mismatch while preserving the quality of modalities. A trustworthy model should output high predictive entropy for these samples. We define the Safe Rejection Rate as the percentage of conflict samples where the predictive entropy of the model exceeds a threshold.

\begin{table}[t]
\centering
\caption{Safety evaluation under Semantic Conflict on MVSA-Single. We compare the ability to correctly reject mismatched pairs. The table reports Safe Rejection Rate, Average Entropy on conflicts, and Conflict Detection Rate (CDR) via AUROC.}
\label{tab:conflict}
\resizebox{\columnwidth}{!}{
\begin{tabular}{l|c|c|c}
\toprule
\textbf{Method} & \textbf{Rejection Rate} $\uparrow$ & \textbf{Avg Entropy} $\uparrow$ & \textbf{CDR (AUROC)} $\uparrow$ \\
\midrule
TMC \cite{han2022trusted} & 15.2\% & 0.45 & 54.3 \\
QMF \cite{zhang2023provable} & 18.5\% & 0.52 & 56.8 \\
PDF \cite{cao2024predictive} & 21.3\% & 0.58 & 60.1 \\
UAW-EEF \cite{guo2025uncertainty} & 26.8\% & 0.72 & 64.5 \\
DBF \cite{bezirganyan2025multimodal} & 35.2\% & 0.94 & 71.2 \\
\midrule
\textbf{GMF (Ours)} & \textbf{76.8\%} & \textbf{1.85} & \textbf{89.4} \\
\bottomrule
\end{tabular}
}
\end{table}

\begin{table}[t]
\centering
\caption{Reliability assessment on PneumoniaMNIST. We report the Pearson correlation ($r$) and the Expected Calibration Error (ECE). GMF provides the most truthful reliability estimates by grounding diagnostic ambiguity in latent transport geometry.}
\label{tab:medical}
\resizebox{\columnwidth}{!}{
\begin{tabular}{l|c|cc}
\toprule
\textbf{Method} & \textbf{Acc (\%)} & \textbf{Correlation ($r$)} & \textbf{ECE ($\downarrow$)} \\
\midrule
QMF \cite{zhang2023provable} & 87.2 & 0.42 & 0.154 \\
PDF \cite{cao2024predictive} & 88.5 & 0.55 & 0.112 \\
UAW-EEF \cite{guo2025uncertainty} & 88.8 & 0.58 & 0.105 \\
DBF \cite{bezirganyan2025multimodal} & 89.1 & 0.61 & 0.095 \\
\midrule
\textbf{GMF (Ours)} & \textbf{91.2} & \textbf{0.78} & \textbf{0.068} \\
\bottomrule
\end{tabular}
}
\end{table}
Tab.~\ref{tab:conflict} highlights the gap between statistical introspection and geometric validation. Evidential methods like QMF and PDF fail to reject conflicts reliably. They simply average two confident but contradictory distributions. This results in blind confidence despite the semantic mismatch. Recent approaches like UAW-EEF and DBF explicitly model disagreement logic. However, they still operate within the prediction space. These methods rely on belief masses derived from the classifier outputs. Consequently, they lack an external signal to contradict the internal certainty of the classifier when both inputs are OOD.

\textbf{Geometric Barrier Verification.} 
GMF achieves a superior rejection rate because it assesses reliability via manifold geometry. Figure~\ref{fig:expr_vis}(b) provides direct confirmation of the Geometric Barrier Principle (Theorem~\ref{thm:barrier}). We observe a clean separation between the clusters:
\begin{itemize}[noitemsep,topsep=0pt]
    \item \textbf{Safe Zone ($E_{\text{inter}} < 5$):} Matched pairs (green dots) concentrate here, retaining high fusion weights.
    \item \textbf{Rejection Zone ($E_{\text{inter}} > 9$):} Conflict pairs (red crosses) are pushed into this high-cost region.
\end{itemize}
Crucially, the fusion weights follow an exponential trend with respect to $E_{\text{inter}}$, up to normalization and constant factors. Equivalently, the gate follows the trend $e^{-E_{\text{inter}}/\kappa}$ before fusion normalization. The strong monotonic relationship between inter-modal cost and fusion weight confirms that the geometric barrier acts as a cross-modal reliability gate, exponentially suppressing modalities that violate semantic consistency.

\subsection{Challenge III: Safety under Diagnostic Ambiguity}
\label{subsec:medical}

\textbf{Hypothesis.} Medical diagnosis presents a distinct challenge. Diagnostic ambiguity arises from genuine edge cases rather than synthetic noise. A pneumonia scan with early infiltrates may trigger high classifier confidence due to spurious correlations. We hypothesize that statistical methods fail to distinguish these confident errors from easy samples. Transport energy will reflect this ambiguity as a high transport correction in latent space. This allows the system to identify samples that lie near decision boundaries despite high predictive confidence.

We evaluate reliability alignment on PneumoniaMNIST using Pearson correlation between the reliability score and predictive correctness. The results in Tab.~\ref{tab:medical} show that GMF significantly outperforms PDF and DBF. This gap indicates that transport energy identifies geometrically marginal cases that statistical calibration misses. Standard methods rely on output entropy. Classifiers often overfit to ambiguous biological features and produce peaked distributions for hard samples. GMF avoids this circular dependency by measuring the learned transport correction in latent space. This transport-energy estimate naturally spikes for atypical presentations.

We validate the clinical utility through risk stratification shown in Figure~\ref{fig:expr_vis}(c). We partition the test set into low, ambiguous, and high-risk tiers based on transport energy quantiles. An ideal system should concentrate all errors in the high-risk tier.
Figure~\ref{fig:expr_vis}(c) reveals that statistical baselines (e.g., QMF, PDF) produce a relatively flat error profile, failing to discriminate truly hard samples. In contrast, GMF achieves high selectivity:
\begin{itemize}[noitemsep,topsep=0pt]
    \item \textbf{Low Risk Tier:} GMF assigns easier samples to this group, supporting safer automation.
    \item \textbf{High Risk Tier:} GMF concentrates difficult samples and most model errors in this group.
\end{itemize}
This separation confirms that our geometric auditing acts as an effective safety valve. Clinically, this enables a triage protocol: automatically approve low-cost scans while deferring the high-cost subset to radiologists, capturing most potential AI errors.

\subsection{Ablation Study and Efficiency Analysis}
\label{subsec:ablation}

We isolate the contributions of the GMF through component replacement. This analysis examines the independence of reliability metrics and computational efficiency, then verify the rectified-transport assumption of trajectory linearity.

\textbf{Geometric vs Statistical Reliability.}
We evaluate the independence of our reliability metric to address the circular dependency. We replace intra-modal transport energy with standard predictive entropy in Tab.~\ref{tab:ablation_main}. This change increases the mutual information between reliability scores and predictive confidence. 

We estimate mutual information using the Kraskov-Stögbauer-Grassberger estimator. The low mutual information for GMF indicates that transport energy is decoupled from classifier outputs. This allows the system to detect noise even when the classifier is overconfident. We also replace flow-based inter-modal cost with linear cosine similarity. The resulting drop in conflict detection confirms that semantic conflicts manifest as non-linear geometric distortions.

\textbf{Efficiency and Trajectory Linearity.}
Our framework utilizes Rectified Flow to linearize transport paths. We compare our one-step velocity prediction against multi-step ODE integration in Tab.~\ref{tab:efficiency}. The negligible difference in accuracy supports using squared initial velocity as the transport-energy estimate. The measured path deviation is small, confirming that learned trajectories are effectively straight. GMF avoids the heavy computational burden of iterative generative models, while maintaining latency parity with lightweight baselines such as PDF and DBF. This ensures the framework is suitable for real-time deployment.

\textbf{Prior Distribution Sensitivity.}
GMF uses a class-agnostic prior as a reference frame for rectified transport. To test sensitivity to this design, we retrained GMF with three priors while keeping all other settings fixed: isotropic Gaussian $\mathcal{N}(0,I)$, anisotropic Gaussian $\mathcal{N}(0,\Sigma)$ using empirical global feature statistics, and a Laplace distribution.

The downstream accuracy and sample-level energy rankings are stable across different priors. Therefore, GMF does not rely on a carefully tuned prior; the main effect comes from the learned transport structure.

\begin{table}[t!]
\centering
\caption{Component replacement study. Gap denotes the performance improvement of GMF over statistical baselines. Low mutual information (MI) confirms that transport energy is decoupled from confidence.}
\label{tab:ablation_main}
\resizebox{\columnwidth}{!}{
\begin{tabular}{l|c|c|c}
\toprule
Component and Metric & Statistical Baseline & Geometric (Ours) & Gap \\
\midrule
1. Reliability Metric & Predictive Entropy & Transport Energy & \\
Accuracy ($\sigma=2.0$) & 36.8\% & 55.2\% & +18.4\% \\
MI with Confidence & 0.67 (Coupled) & 0.08 (Decoupled) & -0.59 \\
\midrule
2. Conflict Metric & Cosine Similarity & Cross-Modal Flow & \\
Conflict AUROC & 72.4 & 89.4 & +17.0 \\
\midrule
3. Fusion Logic & Naive Averaging & Energy Competition & \\
Safe Rejection & 64.1\% & 76.8\% & +12.7\% \\
\bottomrule
\end{tabular}
}
\end{table}

\begin{table}[t!]
\centering
\caption{Sensitivity of GMF to the reference prior distribution.}
\label{tab:prior_accuracy}
\resizebox{\columnwidth}{!}{
\begin{tabular}{lcccc}
\toprule
Prior & NYU Clean & NYU Noise $\sigma=2.0$ & Food-101 Clean & Food-101 Noise $\sigma=2.0$ \\
\midrule
$\mathcal{N}(0,I)$ & 71.9 & 55.2 & 93.1 & 58.7 \\
$\mathcal{N}(0,\Sigma)$ & 71.7 & 54.9 & 92.8 & 58.4 \\
Laplace & 71.6 & 54.8 & 92.9 & 58.3 \\
\bottomrule
\end{tabular}
}
\end{table}

\begin{table}[t!]
\centering
\caption{Sample-level energy ranking consistency across priors.}
\label{tab:prior_ranking}
\resizebox{\columnwidth}{!}{
\begin{tabular}{lcc}
\toprule
Prior & Spearman w.r.t. $\mathcal{N}(0,I)$ & Kendall w.r.t. $\mathcal{N}(0,I)$ \\
\midrule
$\mathcal{N}(0,\Sigma)$ & 0.97 & 0.94 \\
Laplace & 0.96 & 0.93 \\
\bottomrule
\end{tabular}
}
\end{table}

\begin{table}[t!]
\centering
\caption{Efficiency and linearity verification. We compare inference latency against recent state-of-the-art methods. The comparison between GMF integration steps validates the straight-path assumption.}
\label{tab:efficiency}
\resizebox{\columnwidth}{!}{
\begin{tabular}{l|l|c|c}
\toprule
Method & Type & Latency & Acc (Clean) \\
\midrule
MLA \cite{zhang2024multimodal} & Optimization & 156ms & 72.1\% \\
PDF \cite{cao2024predictive} & Evidential & 14ms & 72.5\% \\
DBF \cite{bezirganyan2025multimodal} & Belief Function & 16ms & 72.3\% \\
UAW-EEF \cite{guo2025uncertainty} & Evidential & 15ms & 71.8\% \\
\midrule
GMF (16-Step Integ.) & ODE Solver & 142ms & 72.2\% \\
GMF (1-Step Pred.) & Velocity Proxy & 18ms & 71.9\% \\
\bottomrule
\end{tabular}
}
\end{table}

\section{Conclusion}
\label{sec:conclusion}
Trustworthy multimodal fusion must move beyond self-referential statistical metrics to ensure safe decision-making. We demonstrate that reliability is more effectively assessed as an extrinsic geometric property rather than an internal prediction outcome. This shift to transport dynamics allows systems to identify corrupted modalities that deceive standard uncertainty quantification. Our findings suggest that manifold geometry provides a robust foundation for assessing semantic consistency. We provide a detailed discussion of the computational limitations and future directions of this geometric approach in the Appendix~\ref{app:limitations}.

\section*{Acknowledgments}

This work is supported by the National Science and Technology Major Project of China (No. 2025ZD0219200), the National Natural Science Foundation of China (No. 62406225), and partly supported by Shanghai Science and Technology Committee under Grant No. 24511103900.

\section*{Impact Statement}

This paper presents work whose goal is to advance the field of Machine
Learning. There are many potential societal consequences of our work, none which we feel must be specifically highlighted here.


\bibliography{example_paper}
\bibliographystyle{icml2026}

\appendix
\onecolumn
\section{Limitations}
\label{app:limitations}

Our framework operates strictly on the latent representations produced by upstream encoders~\citep{brahma2015deep, monti2017geometric, lee2018simple}. We do not reinitialize the feature extraction process itself. Consequently, the validity of our geometric reliability relies on Assumption~\ref{ass:regularity}. We assume that the feature extractors map semantically distinct inputs to metric-separable manifolds. If the encoder suffers from mode collapse or fails to disentangle class features, the geometric transport costs may become uninformative. In such cases, GMF outputs a uniform distribution due to the lack of distinct geometric guidance. While this behavior prevents overconfident errors, our method cannot recover semantic information that is lost during the initial feature extraction stage.

Unlike some statistical fusion approaches that are parameter-free, GMF introduces a lightweight velocity network for transport cost estimation. This adds a minor computational step during inference. However, this overhead is negligible in practice. Modern GPU architectures process the small velocity network (1 step) in parallel with high efficiency. The total inference time remains dominated by the CPU-GPU latency. Therefore, the relative latency increase from our geometric auditing step is insignificant for real-time applications.

\section{Implementation Details}
\label{app:implementation}

\medskip
\noindent
\textbf{Remark on the DSB--RF connection.}
GMF instantiates the Diffusion Schr\"{o}dinger Bridge (DSB) via a single-shot Rectified Flow (RF) training procedure without iterative proportional fitting (IPF). This design is supported by the theoretical finding of~\citet{liu2023flow} that a single flow-matching training already produces a valid SB approximation, with reflow serving only to further straighten the trajectory. The approach is also consistent with the diffusion Schr\"{o}dinger bridge matching framework of~\citet{shi2023diffusion}, which similarly adopts single-pass matching to bypass IPF iteration. In our setting, this single-shot strategy preserves the geometric reliability signal while maintaining the low inference latency required for real-time multimodal deployment.

\textbf{Extractor.} Following the experimental protocols established in QMF~\cite{zhang2023provable}, we utilize pre-trained encoders to extract latent representations for each modality. Specifically, we employ ResNet-50/152 for visual modalities and BERT/RoBERTa for textual modalities, depending on the dataset. The extracted features are mapped to a unified latent dimension via a single linear layer before fusion.

\textbf{Velocity Networks.}
To estimate transport costs efficiently, we parameterize the vector fields with lightweight architectures. 
Both intra-modal and inter-modal flows perform diffusion in a shared latent space $\mathbb{R}^{d}$. 
For any modality $m$, we first project its latent feature $z^{(m)} \in \mathbb{R}^{d_m}$ into $\mathbb{R}^{d}$ using a linear projection
$W_{\text{proj}}^{(m)} \in \mathbb{R}^{d_m \times d}$. 
The velocity network then operates in $\mathbb{R}^{d}$ with a two-layer MLP:
$\text{Linear} \to \text{SiLU} \to \text{Linear}$, producing a velocity (noise direction) in $\mathbb{R}^{d}$.

For cross-modal transport $v_{\Phi}^{(a \to b)}$, we follow the same rule. 
We project the source feature $z^{(a)} \in \mathbb{R}^{d_a}$ into the shared space $\mathbb{R}^{d}$ via
$W_{\text{proj}}^{(a)} \in \mathbb{R}^{d_a \times d}$, and apply the same two-layer MLP in $\mathbb{R}^{d}$.
This removes the need to match $d_a$ and $d_b$ inside the velocity network, since diffusion always happens in $\mathbb{R}^{d}$.

This unified design keeps the parameter footprint small. 
As summarized in Tab.~\ref{tab:network_structure}, each velocity network uses only three weight matrices:
one projection matrix and two MLP matrices. 
We avoid heavy attention blocks and deep recurrent backbones, so the transport-cost auditing adds negligible overhead.

\begin{table}[h]
\centering
\caption{Structural specifications of the proposed velocity networks. Both intra-modal and inter-modal flows first project features into a shared diffusion space $\mathbb{R}^{d}$, then apply a lightweight two-layer MLP. ($d$: shared diffusion dim, $d_m$: intra-modal dim, $d_a$: source dim, $\sigma$: SiLU). ()$^*$ denotes the cyclic structure of the diffusion.}
\label{tab:network_structure}
\resizebox{\linewidth}{!}{%
\begin{tabular}{lcccc}
\toprule
\textbf{Module} & \textbf{Mapping} & \textbf{Architecture Flow} & \textbf{Weight Matrices} & \textbf{Dimensions} \\ 
\midrule
Intra-Modal ($v_\theta^{(m)}$) 
& $z^{(m)} \to \text{Noise in } \mathbb{R}^{d}$ 
& $\text{Proj}^{(m)} \to (\text{Linear} \xrightarrow{\sigma} \text{Linear})^*$ 
& 3 
& $W_{\text{proj}}^{(m)} \in \mathbb{R}^{d_m \times d};\ W_{1,2} \in \mathbb{R}^{d \times d}$ \\ 

Inter-Modal ($v_\Phi^{(a \to b)}$) 
& $z^{(a)} \to z^{(b)}$ 
& $\text{Proj}^{(a)} \to (\text{Linear} \xrightarrow{\sigma} \text{Linear})^*$ 
& 3 
& $W_{\text{proj}}^{(a)} \in \mathbb{R}^{d_a \times d};\ W_{1,2} \in \mathbb{R}^{d \times d}$ \\ 
\bottomrule
\end{tabular}%
}
\end{table}

This unified design keeps the parameter footprint small. Taking Food-101 as an example (Input: 2048d/768d, Shared: 512d), the total parameter count is only $\approx$ 4.98M. The initialization projection accounts for $\approx$ 2.88M parameters, while the four velocity bridges (two intra, two inter) add $\approx$ 2.10M parameters.
Computationally, the unified dimension mapping cost is negligible: $\approx$ 0.01 GFLOPs per step. Even with 16-step integration during training, the total geometric overhead is merely $\approx$ 0.16 GFLOPs, and for single-step inference, it remains $\approx$ 0.01 GFLOPs.
We avoid heavy attention blocks and deep recurrent backbones, so the transport-cost auditing adds negligible overhead.

Crucially, feature extractors are pre-trained when constructing the latent representations used by GMF. The geometric loss $\mathcal{L}_{\mathrm{geo}}$ updates only the geometric branch, namely the intra-modal and cross-modal velocity fields; it does not update the classifier/evidential heads or the feature extractors. Conversely, the classification loss $\mathcal{L}_{\mathrm{task}}$ updates only the evidential decision branch and does not update the geometric branch. The conflict-aware regularizer uses geometry-derived quantities as detached reliability signals for the decision branch unless explicitly stated otherwise. Thus, GMF is trained as a unified framework with separated gradient paths, rather than as a fully entangled end-to-end system in which every loss backpropagates through all modules.

\textbf{Hyperparameters and Optimization.} We use the Adam optimizer for all experiments. The initial learning rate is set to $5 \times 10^{-5}$ for the fusion classifier and the velocity networks. We employ a cosine annealing scheduler (at first 20\% epochs) for learning rate decay. We apply a weight decay of $1 \times 10^{-4}$ and stochastic depth of $0.4$ to prevent overfitting. The key hyperparameters unique to our GMF framework are summarized in Tab.~\ref{tab:hyperparameters}.

\begin{table}[ht]
\centering
\caption{Hyperparameter settings for the proposed GMF framework.}
\label{tab:hyperparameters}
\begin{tabular}{lccccccc}
\toprule
\textbf{Parameter} & $\tau$ & $\kappa$ & $\theta_r$ & $\lambda$ & $\zeta $ & $\lambda_{geo}$ & $\lambda_{reg}$ \\ \midrule
\textbf{Value}     & 1.0    & 1.0      & 0.5    & 0.1       &  2 & 1.0             & 0.5             \\ \bottomrule
\end{tabular}
\end{table}

\textbf{Training Setup}
All models are implemented in PyTorch and trained on NVIDIA RTX 4090 GPU. For the Rectified Flow matching, we sample the time step $t \sim \mathcal{U}[0, 1]$ uniformly. The batch size is set to 64 for NYU Depth V2 and 128 for other datasets. We train the models for 20 epochs.

\textbf{Codes.} Given the availability of robust and comprehensive implementations for certain baselines (e.g., QMF~\citep{zhang2023provable}), we have refrained from redundant reimplementation. We provide only the code developed specifically for our reproduction, and direct readers to the official repositories associated with the cited works for the remaining components. This approach serves to honor the original contributors' efforts while strictly adhering to ICML anonymity protocols, which preclude the inclusion of potentially deanonymizing links in the supplementary material.

\subsection{Efficiency and Scaling Analysis}

Since each modality is projected into a fixed-dimensional latent space, the dependence on original feature dimensionality is mainly linear through the projection layers. The inter-modal part scales quadratically with the number of modalities because pairwise cross-modal transport blocks are computed. These blocks are independent and can be computed in parallel. In practice, the absolute overhead remains modest.

\begin{table}[h]
\centering
\caption{Scaling of GMF transport blocks with the number of modalities.}
\label{tab:scaling_modalities}
\small
\begin{tabular}{cccccc}
\toprule
Modalities & Intra Blocks & Inter Blocks & Params & FLOPs & Mean Time (ms) \\
\midrule
3 & 3 & 6 & 9,714,688 & 2,483,159,040 & 1.3275 \\
4 & 4 & 12 & 17,065,472 & 4,362,207,232 & 2.2043 \\
5 & 5 & 20 & 26,516,480 & 6,778,126,336 & 3.3422 \\
\bottomrule
\end{tabular}
\end{table}

\begin{table}[h]
\centering
\caption{Full-pipeline latency and memory on RTX 4090 with ResNet-152 and BERT-base feature extractors kept on GPU.}
\label{tab:pipeline_latency}
\small
\begin{tabular}{lccccc}
\toprule
Pipeline & Batch & Steps & Mean (ms) & Peak Alloc Mean (MiB) & Peak Reserv Mean (MiB) \\
\midrule
GMF & 1 & 1 & 14.7566 & 802.44 & 854.00 \\
GMF & 1 & 16 & 24.2284 & 802.44 & 854.00 \\
MLA & 1 & 1 & 16.9957 & 809.44 & 984.00 \\
QMF & 1 & 1 & 13.6565 & 819.46 & 880.00 \\
GMF & 256 & 1 & 355.0958 & 3382.42 & 4802.00 \\
GMF & 256 & 16 & 360.3813 & 3382.42 & 4804.00 \\
MLA & 256 & 1 & 356.1937 & 3382.42 & 4804.00 \\
QMF & 256 & 1 & 353.4991 & 3382.42 & 4804.00 \\
\bottomrule
\end{tabular}
\end{table}

The 1-step GMF pipeline is comparable to QMF and faster than MLA at batch size 1. At batch size 256, the difference among GMF, MLA, and QMF is small because the feature extractor dominates the runtime. The 16-step GMF variant only adds marginal overhead at large batch size in this setup, supporting the practicality of the rectified one-step design.

\begin{figure*}[h]
    \centering
    \includegraphics[width=0.8\linewidth]{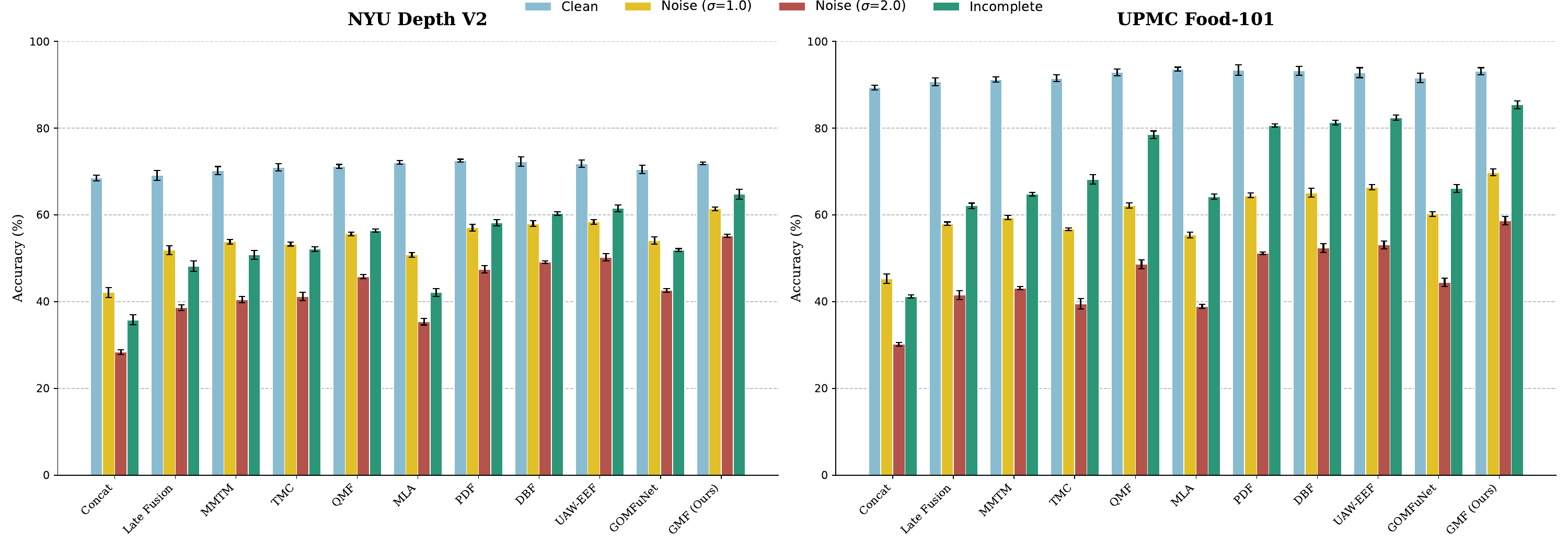}
    \caption{\textbf{Statistical Significance and Robustness Analysis.} 
    Comparative classification accuracy (\%) on NYU Depth V2 and UPMC Food-101 datasets under clean, noisy ($\sigma \in \{1.0, 2.0\}$), and incomplete modalities. 
    Standard deviations are reported via error bars across 3 independent runs. 
    \textbf{GMF (Ours)} consistently achieves superior performance with low variance, particularly in high-noise regimes (red bars), validating the stability of the proposed geometric transport metric.}
    \label{fig:robustness_std}
\end{figure*}

\section{Error Bar}
\label{app:bar}
To further validate the statistical significance of our results, we report the mean accuracy and standard deviation across three independent runs (Seeds: 0, 12, 123) in Figure \ref{fig:robustness_std}. As illustrated, GMF demonstrates superior stability compared to leading baselines like PDF and UAW-EEF. Notably, under severe corruption ($\sigma=2.0$), GMF not only maintains the highest accuracy but also exhibits narrow error margins. This confirms that the performance gains are statistically significant and attributable to the geometric reliability metric, rather than random initialization variance.

\section{Theoretical Proofs and Derivations}
\label{app:theory}

This appendix provides complete proofs for all theoretical results in 
Section~\ref{sec:theory}.

\subsection{Preliminaries and Notation}

Let the latent space be the Euclidean space $\mathbb{R}^d$ equipped with 
the $\ell_2$ norm $\|\cdot\|_2$. 

\textbf{Notation for distances and neighborhoods.}
For a set $\mathcal{S} \subset \mathbb{R}^d$ and radius $r > 0$, we define:
\begin{align}
    \mathrm{dist}(x, \mathcal{S}) 
    &:= \inf_{s \in \mathcal{S}} \|x - s\|_2, 
    \label{eq:def_dist} \\
    \mathcal{B}(\mathcal{S}, r) 
    &:= \{x \in \mathbb{R}^d : \mathrm{dist}(x, \mathcal{S}) \le r\}.
    \label{eq:def_ball}
\end{align}
Here, $\mathrm{dist}(x, \mathcal{S})$ is the distance from point $x$ to 
the set $\mathcal{S}$, and $\mathcal{B}(\mathcal{S}, r)$ is the 
\emph{$r$-neighborhood} of $\mathcal{S}$, also known as the \emph{Minkowski sum} 
$\mathcal{S} \oplus \mathcal{B}(0, r)$.

\textbf{Restatement of Assumption~\ref{ass:regularity}.}
For completeness, we restate the latent space regularity assumption from 
the main text.

\begin{assumption}[Latent Space Regularity]
Let $z^{(m)} \sim P(\cdot \mid y=k)$ denote a latent representation from 
modality $m$ conditioned on class $k \in \{1, \dots, K\}$. For each modality, there exist 
class-specific geometric structures $\{\mathcal{M}_k^{(m)}\}_{k=1}^K 
\subset \mathbb{R}^d$ satisfying the following two properties:
\begin{enumerate}[label=(\roman*), leftmargin=2em]
    \item \textbf{Concentration.} With high probability (at least $1-\beta$ 
    for small $\beta > 0$), samples lie within an $\epsilon$-neighborhood 
    of their class structure:
    \[
    P\big( z^{(m)} \in \mathcal{B}(\mathcal{M}_k^{(m)}, \epsilon) \big) \ge 1 - \beta.
    \]
    
    \item \textbf{Metric Separability.} Distinct class structures are 
    separated by a margin $\delta > 2\epsilon$:
    \[
    \forall m,\ \forall i \neq j, \quad 
    \inf_{u \in \mathcal{M}_i^{(m)},\, v \in \mathcal{M}_j^{(m)}} \|u - v\|_2 \ge \delta.
    \]
\end{enumerate}
\end{assumption}

\noindent
\textbf{Interpretation.}
Condition (i) states that well-trained encoders produce representations 
that concentrate near class-specific geometric structures. The parameter 
$\epsilon$ captures intra-class variability (noise, augmentation, etc.). 
Condition (ii) ensures that these structures are well-separated, with the 
constraint $\delta > 2\epsilon$ guaranteeing that the $\epsilon$-neighborhoods 
$\mathcal{B}(\mathcal{M}_i^{(m)}, \epsilon)$ and $\mathcal{B}(\mathcal{M}_j^{(m)}, \epsilon)$ 
are disjoint for $i \neq j$.

\subsection{Proof of Theorem~\ref{thm:optimality} (Optimality of Fusion Weights)}

\textbf{Statement.}
We seek to minimize the entropy-regularized geometric cost over the 
probability simplex $\Delta^{M-1}$:
\begin{equation}
    \min_{w \in \Delta^{M-1}} \; \mathcal{J}(w) 
    := \sum_{m=1}^M w^{(m)} \mathcal{C}^{(m)} - \tau H(w),
    \label{eq:optimization_objective}
\end{equation}
where $\Delta^{M-1} := \{w \in \mathbb{R}^M : \sum_{m=1}^M w^{(m)} = 1,\, w^{(m)} \ge 0 \,\forall m\}$ 
is the $(M-1)$-dimensional probability simplex, and 
$H(w) := -\sum_{m=1}^M w^{(m)} \ln w^{(m)}$ is the Shannon entropy with the standard convention $0\ln 0=0$.

\begin{proof}
The objective $\mathcal{J}(w)$ is strictly convex on $\Delta^{M-1}$ 
because the negative entropy $-H(w) = \sum_m w^{(m)} \ln w^{(m)}$ is 
strictly convex (its Hessian is $\text{diag}(1/w^{(1)}, \dots, 1/w^{(M)})$, 
which is positive definite for $w^{(m)} > 0$). The feasible region 
$\Delta^{M-1}$ is a compact convex set. By the extreme value theorem, 
a global minimizer exists. Strict convexity guarantees uniqueness.

The entropy barrier yields an interior stationary point for $\tau>0$; the KKT solution below has strictly positive weights. Thus, at the optimum, the non-negativity constraints are inactive and we enforce only the equality constraint $\sum_{m=1}^M w^{(m)} = 1$.

We form the Lagrangian:
\begin{equation}
    \mathcal{L}(w, \lambda) 
    := \sum_{m=1}^M w^{(m)} \mathcal{C}^{(m)} 
       + \tau \sum_{m=1}^M w^{(m)} \ln w^{(m)} 
       + \lambda \left( \sum_{m=1}^M w^{(m)} - 1 \right),
    \label{eq:lagrangian}
\end{equation}
where $\lambda \in \mathbb{R}$ is the Lagrange multiplier.

Taking the partial derivative with respect to $w^{(m)}$ and setting it 
to zero (KKT stationarity condition):
\begin{equation}
    \frac{\partial \mathcal{L}}{\partial w^{(m)}} 
    = \mathcal{C}^{(m)} + \tau (\ln w^{(m)} + 1) + \lambda = 0.
    \label{eq:kkt_condition}
\end{equation}

Solving for $\ln w^{(m)}$:
\[
\tau \ln w^{(m)} = -\mathcal{C}^{(m)} - \tau - \lambda
\quad \Longrightarrow \quad
\ln w^{(m)} = -\frac{\mathcal{C}^{(m)}}{\tau} - \frac{\lambda + \tau}{\tau}.
\]

Exponentiating both sides:
\begin{equation}
    w^{(m)} = \exp\left( -\frac{\mathcal{C}^{(m)}}{\tau} \right) 
              \cdot \underbrace{\exp\left( -\frac{\lambda + \tau}{\tau} \right)}_{=: Z^{-1}},
    \label{eq:unnormalized_weight}
\end{equation}
where $Z := \exp\left( \frac{\lambda + \tau}{\tau} \right)$ is a constant 
independent of $m$.
Imposing the constraint $\sum_{m=1}^M w^{(m)} = 1$:
\[
\sum_{m=1}^M \left[ \frac{1}{Z} \exp\left( -\frac{\mathcal{C}^{(m)}}{\tau} \right) \right] = 1
\quad \Longrightarrow \quad
Z = \sum_{j=1}^M \exp\left( -\frac{\mathcal{C}^{(j)}}{\tau} \right).
\]

Substituting $Z$ back into Eq.~\eqref{eq:unnormalized_weight} yields the 
optimal weights:
\begin{equation}
    w^{*(m)} = \frac{\exp\left( -\frac{\mathcal{C}^{(m)}}{\tau} \right)}
                    {\sum_{j=1}^M \exp\left( -\frac{\mathcal{C}^{(j)}}{\tau} \right)}.
    \label{eq:optimal_weights_app}
\end{equation}

This is the unique global minimizer of~\eqref{eq:optimization_objective}.
\end{proof}

\subsection{Proof of Theorem~\ref{thm:barrier} (Geometric Barrier Principle)}
\label{app:proof_barrier}

\textbf{Statement.}
We establish a lower bound on the cross-modal transport cost when 
modalities encode conflicting classes.

\textbf{Setup and assumptions.}
Consider the following scenario:
\begin{enumerate}[label=(\roman*), leftmargin=2em]
    \item \textbf{Source modality $n$.} Encodes the ground-truth class $y$ 
    with representation $z^{(n)} \in \mathcal{B}(\mathcal{M}_y^{(n)}, \epsilon)$.
    
    \item \textbf{Target modality $B$.} Encodes a conflicting class $k \neq y$ 
    with representation $z^{(B)} \in \mathcal{B}(\mathcal{M}_k^{(B)}, \epsilon)$.
    
    \item \textbf{Cross-modal mapping $\Phi_{n \to B}$.} Here $\Phi_{n\to B}(z):=z+v_{\Phi}^{(n\to B)}(z,0)$ is the one-step map induced by the learned cross-modal velocity field. We assume this
    mapping is \emph{$\xi$-semantically consistent} with $\xi \le \epsilon$. 
    Formally, for any point $u \in \mathcal{B}(\mathcal{M}_y^{(n)}, \epsilon)$,
    \begin{equation}
        \mathrm{dist}\big(\Phi_{n \to B}(u), \mathcal{M}_y^{(B)}\big) \le \xi.
        \label{eq:semantic_consistency}
    \end{equation}
    This means the mapping preserves class identity: points near the class-$y$ 
    structure in modality $n$ are mapped near the class-$y$ structure in 
    modality $B$, with at most $\xi$ distortion.
\end{enumerate}

\begin{proof}
Let $\tilde{z} := \Phi_{n \to B}(z^{(n)})$ denote the mapped representation.

\medskip
\noindent
By assumption (i), $z^{(n)} \in \mathcal{B}(\mathcal{M}_y^{(n)}, \epsilon)$. 
Applying the semantic consistency property~\eqref{eq:semantic_consistency}, 
we have:
\[
\mathrm{dist}\big(\tilde{z}, \mathcal{M}_y^{(B)}\big) \le \xi.
\]
This implies there exists a point $p_y \in \mathcal{M}_y^{(B)}$ such that:
\begin{equation}
    \|\tilde{z} - p_y\|_2 \le \xi.
    \label{eq:tilde_z_bound}
\end{equation}

\medskip
\noindent
By assumption (ii), $z^{(B)} \in \mathcal{B}(\mathcal{M}_k^{(B)}, \epsilon)$. 
Thus, there exists a point $p_k \in \mathcal{M}_k^{(B)}$ such that:
\begin{equation}
    \|z^{(B)} - p_k\|_2 \le \epsilon.
    \label{eq:z_A_bound}
\end{equation}

\medskip
\noindent
We aim to lower-bound $\|\tilde{z} - z^{(B)}\|_2$. By the triangle inequality:
\begin{equation}
    \|\tilde{z} - z^{(B)}\|_2 + \|\tilde{z} - p_y\|_2 + \|z^{(B)} - p_k\|_2 
    \ge \|p_y - p_k\|_2.
    \label{eq:triangle_ineq}
\end{equation}

Rearranging terms:
\[
\|\tilde{z} - z^{(B)}\|_2 
\ge \|p_y - p_k\|_2 - \|\tilde{z} - p_y\|_2 - \|z^{(B)} - p_k\|_2.
\]

\medskip
\noindent
By Assumption~\ref{ass:regularity}(ii) (Metric Separability), since $k \neq y$:
\begin{equation}
    \|p_k - p_y\|_2 \ge \delta.
    \label{eq:manifold_separation}
\end{equation}

Combining Eqs.~\eqref{eq:tilde_z_bound}, \eqref{eq:z_A_bound}, 
and~\eqref{eq:manifold_separation}:
\[
\|\tilde{z} - z^{(B)}\|_2 \ge \delta - \xi - \epsilon.
\]

Since we assume $\xi \le \epsilon$:
\begin{equation}
    \|\tilde{z} - z^{(B)}\|_2 \ge \delta - 2\epsilon.
    \label{eq:distance_lower_bound}
\end{equation}

\medskip
\noindent
The cross-modal transport cost is defined using the induced one-step map:
\[
\mathcal{E}_{\mathrm{inter}}^{(n \to B)} 
:= \|\Phi_{n \to B}(z^{(n)}) - z^{(B)}\|_2^2 
= \|\tilde{z} - z^{(B)}\|_2^2.
\]

Squaring both sides of~\eqref{eq:distance_lower_bound}:
\begin{equation}
    \mathcal{E}_{\mathrm{inter}}^{(n \to B)} \ge (\delta - 2\epsilon)^2.
    \label{eq:barrier_bound}
\end{equation}

By Assumption~\ref{ass:regularity}(ii), $\delta > 2\epsilon$, so 
$(\delta - 2\epsilon)^2 > 0$. This establishes a strict positive lower bound 
on the transport cost.
\end{proof}

\medskip
\noindent
\textbf{Remark on the semantic consistency assumption.}
The $\xi$-semantic consistency property~\eqref{eq:semantic_consistency} 
is an idealized local regularity assumption. In GMF, the direction-specific map is learned from matched multimodal pairs by Eq.~\eqref{eq:loss_inter}, which encourages compatible representations to align locally. This empirical objective motivates the assumption, but the pointwise condition remains stronger than the training loss itself.

\subsection{Proof of Corollary}
\label{app:proof_corollary}

\begin{proof}
Let $D=(\delta-2\epsilon)^2$. We prove the same bound as Corollary~\ref{cor:suppression}:
\[
\gamma_{\mathrm{int}}^{(B)}
\le
\lambda(M-1)\exp\left(-\frac{D}{\kappa}\right).
\]
We then account for the numerical floor in the stabilized gate and derive the corresponding suppression of the unnormalized contribution.

\medskip
\noindent
By Eq.~\eqref{eq:interaction}, the interaction gate for modality $B$ is
\begin{equation}
    \gamma_{\mathrm{int}}^{(B)} 
    = \lambda \sum_{n \neq B} r^{(n)}
       \exp\left( -\frac{\mathcal{E}_{\mathrm{inter}}^{(n \to B)}}{\kappa} \right),
    \quad
    r^{(n)}=\sigma(\theta_r-\mathcal{E}_{\mathrm{intra}}^{(n)}).
    \label{eq:interaction_gate_app}
\end{equation}

For every active reliable neighbor $n$ of $B$, Theorem~\ref{thm:barrier} gives
\[
\mathcal{E}_{\mathrm{inter}}^{(n \to B)} \ge D.
\]
Thus,
\[
\exp\left( -\frac{\mathcal{E}_{\mathrm{inter}}^{(n \to B)}}{\kappa} \right) 
\le \exp\left( -\frac{D}{\kappa} \right).
\]

Since $0 \le r^{(n)} \le 1$,
\begin{equation}
    \gamma_{\mathrm{int}}^{(B)} 
    = \lambda \sum_{n \neq B} r^{(n)}
       \exp\left( -\frac{\mathcal{E}_{\mathrm{inter}}^{(n \to B)}}{\kappa} \right)
    \le \lambda (M-1) \exp\left( -\frac{D}{\kappa} \right).
    \label{eq:gamma_upper_bound}
\end{equation}

\medskip
\noindent
Since $\tilde{\gamma}_{\mathrm{int}}^{(B)}=\gamma_{\mathrm{int}}^{(B)}+\epsilon_\gamma$,
\begin{equation}
    \tilde{\gamma}_{\mathrm{int}}^{(B)}
    \le
    \lambda(M-1)\exp\left( -\frac{D}{\kappa} \right)+\epsilon_\gamma.
    \label{eq:tilde_gamma_upper_bound}
\end{equation}

\medskip
\noindent
Define the unnormalized contribution
\[
s_B
= \exp\left(-\frac{\mathcal{E}_{\mathrm{intra}}^{(B)}}{\tau}\right)
\tilde{\gamma}_{\mathrm{int}}^{(B)}.
\]
Since $\mathcal{E}_{\mathrm{intra}}^{(B)}\ge 0$, we have $\exp(-\mathcal{E}_{\mathrm{intra}}^{(B)}/\tau)\le 1$. Therefore,
\begin{equation}
    s_B
    \le
    \tilde{\gamma}_{\mathrm{int}}^{(B)}
    \le
    \lambda(M-1)\exp\left( -\frac{D}{\kappa} \right)+\epsilon_\gamma,
    \label{eq:score_upper_bound}
\end{equation}
which proves exponential suppression of the unnormalized contribution up to the numerical floor.

Finally, assume that the denominator contains reliable non-conflicting evidence with total unnormalized contribution at least $c_0>0$, i.e.,
\[
\sum_{j\ne B}s_j \ge c_0.
\]
Then
\[
w^{*(B)}
=
\frac{s_B}{s_B+\sum_{j\ne B}s_j}
\le
\frac{s_B}{c_0}
\le
\frac{\lambda(M-1)}{c_0}\exp\left( -\frac{D}{\kappa} \right)+\frac{\epsilon_\gamma}{c_0}.
\]
\end{proof}

\end{document}